\documentclass{article}

\usepackage{microtype}
\usepackage{graphicx}
\usepackage{subfigure}
\usepackage{booktabs} 

\usepackage{hyperref}
\usepackage{amsmath}
\usepackage{mathtools} 
\usepackage{amssymb,amsmath,amsthm}
\newtheorem{theorem}{Theorem}

\newtheorem{lemma}[theorem]{Lemma}

\usepackage[accepted]{icml2023}

\icmltitlerunning{MG-GNN: Multigrid Graph Neural Networks for Learning Multilevel Domain Decomposition Methods}

\begin{document}

\twocolumn[
\icmltitle{MG-GNN: Multigrid Graph Neural Networks for Learning Multilevel Domain Decomposition Methods} 




\begin{icmlauthorlist}

  \icmlauthor{Ali Taghibakhshi}{uiucmechse}
  \icmlauthor{Nicolas Nytko}{uiuccs}
  \icmlauthor{Tareq Uz Zaman}{munsci}
  \icmlauthor{Scott MacLachlan}{munscott}
  \icmlauthor{Luke N. Olson}{uiuccs}
  \icmlauthor{Matt West}{uiucmechse}
\end{icmlauthorlist}
\icmlaffiliation{uiucmechse}{Department of Mechanical Science and Engineering, University of Illinois at Urbana-Champaign}
\icmlaffiliation{uiuccs}{Department of Computer Science, University of Illinois at Urbana-Champaign}
\icmlaffiliation{munsci}{Scientific Computing Program, Memorial University of Newfoundland}
\icmlaffiliation{munscott}{Department of Mathematics and Statistics, Memorial University of Newfoundland}

\icmlcorrespondingauthor{Ali Taghibakhshi}{alit2@illinois.edu}

\icmlkeywords{Domain Decomposition, Optimized Schwarz Methods, Graph Neural Networks, Unsupervised Learning}

\vskip 0.3in
]



\printAffiliationsAndNotice{} 

\begin{abstract}
  Domain decomposition methods (DDMs) are popular solvers for discretized systems of partial differential equations (PDEs), with one-level and multilevel variants.  These solvers rely on several algorithmic and mathematical parameters, prescribing overlap, subdomain boundary conditions, and other properties of the DDM.  While some work has been done on optimizing these parameters, it has mostly focused on the one-level setting or special cases such as structured-grid discretizations with regular subdomain construction. In this paper, we propose multigrid graph neural networks (MG-GNN), a novel GNN architecture for learning optimized parameters in two-level DDMs\@. We train MG-GNN using a new unsupervised loss function, enabling effective training on small problems that yields robust performance on unstructured grids that are orders of magnitude larger than those in the training set. We show that MG-GNN outperforms popular hierarchical graph network architectures for this optimization and that our proposed loss function is critical to achieving this improved performance.
\end{abstract}

\section{Introduction}

Differential equations are at the core of many important scientific and engineering problems~\cite{gholizadeh2021evaluation, han2011elliptic}, and often, there is no analytical solution available; hence, researchers utilize numerical solvers~\cite{ vos2002navier, gholizadeh2023suitability}. Among numerical methods for solving the systems of equations obtained from discretization of partial differential equations (PDEs), domain decomposition methods (DDMs) are a popular approach~\cite{MR2104179, MR1857663, MR3450068}.  They have been extensively studied and applied to elliptic boundary value problems, but are also considered for time-dependent problems.  Schwarz methods are among the simplest and most popular types of DDM, and map well to MPI-style parallelism, with both one-level and multilevel variants.
One-level methods decompose the global problem into multiple subproblems (subdomains), which are obtained either by discretizing the same PDE over a physical subdomain or by projection onto a discrete basis, using subproblem solutions to form a preconditioner for the global problem. Classical Schwarz methods generally consider Dirichlet or Neumann boundary conditions between the subdomains, while Optimized Schwarz methods (OSM)~\cite{GHN_optimized_2000} consider a combination of Dirichlet and Neumann boundary conditions, known as Robin-type boundary conditions, to improve the convergence of the method. Restricted additive Schwarz (RAS) methods~\cite{cai1999restricted} are a common form of Schwarz methods, and optimized versions of one-level RAS has been theoretically studied by~\citet{st2007optimized}.  Two-level methods extend one-level approaches by adding a (global) coarse-grid correction step to the preconditioner, generally improving performance but at an added cost.

In recent years, there has been a growing focus on using machine learning (ML) methods to learn optimized parameters for iterative PDE solvers, including DDM and algebraic multigrid (AMG). In~\citet{greenfeld2019learning} convolutional neural networks (CNNs) are used to learn the interpolation operator in AMG on structured problems, and in a following study~\cite{luz2020learning}, graph neural networks (GNNs) are used to extend the results to arbitrary unstructured grids. In a different fashion, reinforcement learning methods along with GNNs are used to learn coarse-gird selection in reduction-based AMG in~\citet{taghibakhshi2021optimization}.
As mentioned in~\citet{heinlein2021combining}, when combining ML methods with DDM, approaches can be categorized into two main families, namely using ML within a classical DDM framework to obtain improved convergence and using deep neural networks as the main solver or discretization module for DDMs. In a recent study~\cite{taghibakhshi2022learning}, GNNs are used to learn interface conditions in optimized Schwarz DDMs that can be applied to many subdomain problems, but their study is limited to one-level solvers. Two-level domain decomposition methods often converge significantly faster than one-level methods since they include coarse-grid correction, but obtaining optimized multilevel DDMs for  general unstructured problems with arbitrary subdomains remains an open challenge.


Graph neural networks (GNNs) extend learning based methods and convolution operators to unstructured data. Similar to structured problems, such as computer vision tasks, many graph-based problems require information sharing beyond just a limited local neighborhood in a given graph. However, unlike in CNNs, where often deep CNNs are used with residual skip connections to achieve long range information passing, GNNs dramatically suffer depth limitations. Stacking too many GNN layers results in \textit{oversmoothing}, which is due to close relation of graph convolution operators to Laplacian smoothing~\cite{li2018deeper, oono2019graph}. Oversmoothing essentially results in indistinguishable node representations after too many GNN layers, due to information aggregation in a large local neighborhood. Inspired by the Unet architecture in CNNs~\cite{ronneberger2015u}, graph U-nets~\cite{gao2019graph} were introduced as a remedy for longe-range information sharing in graphs without using too many GNN layers. Similar to their CNN counterparts, graph-Unet architectures apply down-sampling layers (pooling) to aggregate node information to a coarser representation of the problem with fewer nodes. This is followed by up-sampling layers (unpooling, with the same number of layers as pooling) to reconstruct finer representations of the problem and allow information to flow back to the finer levels from the coarser ones.

As mentioned in~\citet{ke2017multigrid}, U-net and graph-Unet architectures suffer from a handful of problems and non-optimalities. In these architectures, scale and abstraction are combined, meaning earlier, finer layers cannot access the information of the coarse layers. In other words, initial layers learn deep features only based on a local neighborhood without considering the larger picture of the problem. Moreover, finer levels do not benefit from information updates until the information flow reaches the coarsest level and flows back to the finer levels.  That is, the information flow has to complete a full (graph) U-net cycle to update the finest level information, which could potentially require multiple conventional layers, leading to oversmoothing in the case of graph U-nets. More recently, there has been similar hierarchical GNN architectures utilized for solving PDEs, such as those proposed by~\citet{fortunato2022multiscale} and~\citet{li2020multipole}. In each case, the architecture is similar to a U-net, in terms in terms of information flow (from the finest to coarsest graph and back), and there is no cross-scale information sharing, making them prone to the aforementioned U-net type problems.

To fully unlock the ability of GNNs to learn optimal DDM operators, and to mitigate the shortcomings of graph U-nets mentioned above, we introduce here a novel GNN architecture, multigrid graph neural networks (MG-GNN). MG-GNN information flow is parallel at all scales, meaning every MG-GNN layer processes information from both coarse and fine scale graphs. We employ this MG-GNN architecture to advance DDM-based solvers by developing a learning-based approach for two-level optimized Schwarz methods. Specifically, we learn the Robin-type subdomain boundary conditions needed in OSM as well as the overall coarse-to-fine interpolation operator. We also develop a novel loss function essential for achieving superior performance compared to previous two-level optimized RAS\@. The summary of contributions of this paper is as follows:
\begin{itemize}
\item Introduce a multigrid graph neural network (MG-GNN) architecture that outperforms existing hierarchical GNN architectures and scales linearly with problem size;
\item Improve the loss function with theoretical guarantees essential for training two-level Schwarz methods;
\item Enforce scalability, leading to effective training on small problems and generalization to problems that are orders of magnitude larger; and
\item Outperform classical two-level RAS, both as stationary algorithm and as a preconditioner for the flexible generalized minimum residual (FGMRES) iteration.
\end{itemize}

\section{Background}

In this section, we review one and two-level DDMs. Let $\Omega$ be an open set in $\mathbb{R}^{2}$, and consider the Poisson equation:
\begin{equation}\label{eq:Poisson}
-\Delta\Phi = f,
\end{equation}
where $\Delta$ is the Laplace operator and $f(x,y)$ and $\Phi(x,y)$ are real-valued functions. Alongside~\eqref{eq:Poisson}, we consider inhomogeneous Dirichlet conditions on the boundary of $\Omega$, $\partial\Omega$, and use a piecewise linear finite-element (FE) discretization on arbitrary triangulations of $\Omega$. In the linear FE discretization, every node in the obtained graph corresponds to a degree of freedom (DoF) in the discretization, and the set of all nodes is denoted by $D$. The set $D$ is decomposed into $S$ non-overlapping subdomains $\{D_{1}^{0}, D_{2}^{0}, \ldots, D_{S}^{0}\}$ (where the superscript in the notation indicates the amount of overlap; hence, the superscript zero for the non-overlapping decomposition). The union of the subdomains covers the set of all DoFs, $D = \cup D_{i}^{0}$, so that each node in $D$ is contained in exactly one $D_{i}^{0}$. Denote the restriction operator for discrete DoFs onto those in $D_{i}^{0}$ by $R_{i}^{0}$ and the corresponding extension from $D_{i}^{0}$ to $D$ by $(R_{i}^{0})^{T}$. Following the FE discretization of problem, we obtain a linear system to solve, $Ax=b$, where $A$ is the global stiffness matrix. For every $D_{i}^{0}$, we obtain the subdomain stiffness matrix as $A_{i}^{0} = R_{i}^{0}A(R_{i}^{0})^{T}$. In the OSM setting, alternative definitions to this Galerkin projection for $A_{i}^{0}$ are possible as noted below. To obtain the coarse-level representation of the problem, let $P\in\mathbb{R}^{S\times|D|}$ be the piecewise-constant interpolation operator that assigns every node in $D_{i}^{0}$ to the $i$-th coarse node. The coarse-level operator is then obtained as $A_{C} = P^TAP$.

The restricted additive Schwarz method (RAS)~\cite{cai1999restricted} is an important extension to the Schwarz methodology for the case of overlapping subdomains, where some nodes in $D$ belong to more than one subdomain. Denoting the overlap amount by $\delta\in\mathbb{N}$, we define the subdomains $D_{i}^{\delta}$ for $\delta>0$ by recursion, as $D_{i}^{\delta} = D_{i}^{\delta-1}\cup\{j\,|\,a_{kj}\neq0\text{\;for\;}k\in D_{i}^{\delta-1}\}$. For the coarse-grid interpolation operator, $P$, each of the overlapping nodes is now associated with multiple columns of $P$, which is typically chosen as a partition of unity, with rows of $P$ having equal non-zero weights (that can be interpreted as the probability of assigning a fine node to a given subdomain).
The conventional two-level RAS preconditioner is then defined by considering the fine-level operator, $M_{\text{RAS}}$, and the coarse-level correction operator, $C_{\text{2-RAS}}$, given by

\begin{align}
M_{\text{RAS}} = \sum\limits_{i=1}^{S} (\tilde{R}^{\delta}_{i})^{T}(A_{i}^{\delta})^{-1}{R_{i}^{\delta}}\label{eq:interface},\\
C_{\text{2-RAS}} = P(P^{T}AP)^{-1}P^{T}\label{eq:cgsolve},
\end{align} where $A_{i}^{\delta} = {(R_{i}^{\delta})}^{T}AR_{i}^{\delta}$. The operator $R_{i}^{\delta}$ denotes restriction for DoFs in D to those in $D_{i}^{\delta}$ while $\tilde{R}^{\delta}_{i}$ is a modified restriction from $D$ to $D_{i}^{\delta}$ that takes nonzero values only for DoFs in $D_{i}^{0}$.  The two-level RAS preconditioner is given as $M_{\text{2-RAS}} = C_{\text{2-RAS}} + M_{\text{RAS}} - C_{\text{2-RAS}}AM_{\text{RAS}}$, with the property that $I-M_{\text{2-RAS}}A = (I-C_{\text{2-RAS}}A)(I-M_{\text{RAS}}A)$.

In the case of optimized Schwarz, the subdomain systems (fine-level $A_{i}^{\delta}$) are modified by imposing a Robin boundary condition between subdomains, writing $\tilde{A}_{i}^{\delta} = A_i^{\delta} +L_{i}$, where  $L_i$ is the term resulting from the Robin-type condition:
\begin{align}
  \alpha u + \vec{n}\cdot\nabla u & = g(x),
\end{align}
where $g$ denotes inhomogeneous data and $\vec{n}$ is the outward unit normal to the boundary. The fine-level operator for optimized Schwarz is then given by
\begin{equation}\label{eq:M_ORAS}
M_{\text{ORAS}} = \sum\limits_{i = 1}^S\left(\tilde{R}_{i}^\delta\right)^{T}\left(\tilde{A}_{i}^{\delta}\right)^{-1}R_{i}^{\delta},
\end{equation}
where the choice of weight, $\alpha$, in the subdomain Robin boundary condition is a parameter for optimization.  Similarly, the method can be improved by optimizing the choice of coarse-level interpolation operator, $P$, but this has not been fully explored in the OSM literature.  Similarly to with RAS, we define the two-level ORAS preconditioner as $M_{\text{2-ORAS}} = C_{\text{2-RAS}} + M_{\text{ORAS}} - C_{\text{2-RAS}}AM_{\text{ORAS}}$.

The work of~\citet{taghibakhshi2022learning} suggests a method to learn $L_{i}$ for one-level ORAS\@. Here, we learn both $L_{i}$ and $P$ for two-level methods since, as later shown in Figure~\ref{fig:level-ablation-stationary}, the two-level methods are significantly more robust. Furthermore, as we show in Section~\ref{subsec:learningLP}, while learning both ingredients improves the performance, learning the interpolation operator, $P$, is significantly more important than learning $L_{i}$'s in order to obtain a two-level solver that outperforms classical two-level RAS\@.

\section{Multigrid graph neural network}\label{sec:mg-gnn}

The multigrid neural architecture~\cite{ke2017multigrid} is an architecture for CNNs that extracts higher level information in an image more efficiently by cross-scale information sharing, in contrast to other CNN architectures, such as U-nets, where abstraction is combined with scale. That is, in one multigrid layer, the information is passed between different scales of the problem, removing the necessity of using deep CNNs or having multilevel U-net architectures. Inspired by~\citet{ke2017multigrid}, we develop a multigrid architecture for GNNs, enabling cross-scale message (information) passing without making the GNN deeper; we call our architecture Multigrid GNN, or MG-GNN\@. Figure~\ref{fig:MG-GNN} shows one layer of the MG-GNN with two levels (a fine and a coarse level).

The input data to one layer of an MG-GNN has $L$ different graphs, from fine to coarse, denoted by $G^{(\ell)} = (X^{(\ell)}, A^{(\ell)})$, where $A^{(\ell)}\in\mathbb{R}^{n_\ell\times n_\ell}$ and $X^{(\ell)}\in\mathbb{R}^{n_{\ell}\times d}$ are adjacency and node feature matrices, respectively, and $n_\ell$ and $d$ are the number of nodes and node feature dimension in $\ell$-th graph for $\ell\in\{0,1, \ldots,L-1\}$, with $\ell=0$ denoting the finest level. If the input graph does not have multiple levels, we obtain the coarser levels recursively by considering a node assignment matrix (clustering operator) $R^{(\ell)}\in\mathbb{R}^{n_{\ell+1}\times n_{\ell}}$, for $\ell\in\{0,1, \ldots,L-2\}$:

\begin{align}
X^{(\ell+1)} = R^{(\ell)}X^{(\ell)},\label{eq:next_feature}\\
A^{(\ell+1)} = R^{(\ell)}A^{(\ell)}(R^{(\ell)})^{T}.\label{eq:next_A}
\end{align}

We note that, in general, the assignment matrix $R^{\ell}$ could be any pooling/clustering operator, such as $k$-means clustering, learnable pooling, etc. We denote $R^{(\ell\rightarrow k)}$ to be the assignment matrix of graph level $\ell$ to level $k$ (with $\ell < k$), which is constructed through $R^{(\ell\rightarrow k)} = \prod\limits_{j=\ell}^{k-1}R^{(j)}$ (down-sampling). To complement this terminology, we also define $R^{(k\rightarrow \ell)}= (R^{(\ell\rightarrow k)})^{T}$ for $\ell>k$ (up-sampling), and for the case of $\ell=k$, the assignment matrix is simply the identity matrix of dimension $n_\ell$.  The mathematical formalism of the $m$-th layer of the MG-GNN with $L$ levels is as follows: given all graphs feature matrices, $X_{m}^{(\ell)}$, for $\ell\in\{0,1,\ldots,L-1\}$:
\begin{align}
\dot{X}^{\ell\rightarrow k} = F^{\ell\rightarrow k}(X^{(\ell)}_{m}, X^{(k)}_{m}, R^{(\ell\rightarrow k)})\label{eq:fij}\\
\tilde{X}^{(\ell)}_{m} = [\dot{X}^{0\rightarrow \ell} \| \dot{X}^{1\rightarrow \ell} \| \ldots \| \dot{X}^{k-1\rightarrow \ell}]\\
X^{(\ell)}_{m+1} = \text{GNN}^{(\ell)}(\tilde{X}^{(\ell)}_{m}, A^{(\ell)})\label{eq:gnni}
\end{align}
where $\|$ denotes concatenation, and $\text{GNN}^{(\ell)}$ and $F^{\ell\rightarrow \ell}$ could be any homogeneous and heterogeneous GNNs, respectively. For the case of $\ell\ne k$, we consider $F^{\ell\rightarrow k}$ to be a heterogeneous message passing scheme between levels $\ell$ and $k$, which is defined as follows. Consider any node $v$ in $G^{(\ell)}$ and denote the row in $X^{(\ell)}_{m}$ corresponding to the feature vector of node $v$ by $x_v$. Then, $F^{\ell\rightarrow k}(X^{(\ell)}_{m}, X^{(k)}_{m}, R^{(\ell\rightarrow k)})$ is given by
\begin{align}
m_v=g^{\ell\rightarrow k}\left(\underset{{\omega}\in\mathcal{N}(v)}{\square}f^{\ell\rightarrow k}(x_v, x_{\omega}, e_{v\omega}), x_v\right)\label{eq:mpnn1}
\end{align}
where $e_{v\omega}$ is the feature vector of the edge (if any) connecting $v$ and $\omega$, $\square$ is any permutation invariant operator such as sum, max, min, etc., and $f^{\ell\rightarrow k}$ and $g^{\ell\rightarrow k}$ are learnable multilayer perceptrons (MLPs). See Figure~\ref{fig:UDsampling} for visualization of up-sampling and down-sampling in MG-GNN\@.
In this study, we consider a two-level MG-GNN (see Figure~\ref{fig:2Lmggnn}) and, for the clustering, we consider a $k$-means-based clustering algorithm (best known as Lloyd's algorithm) which has $O(n)$ time complexity and guarantees that every node will be assigned to a subdomain~\cite{bell2008algebraic, lloyd1982least}) in a connected graph.  As mentioned earlier, the MG-GNN architecture could alternatively use any pooling/clustering method such as DiffPool~\cite{ying2018hierarchical}, top-$K$ pooling~\cite{gao2019graph}, ASAP~\cite{ranjan2020asap}, SAGPool~\cite{lee2019self}, to name but a few. However, for the case of this paper, since RAS (and therefore, ORAS) necessitates every node in the fine grid be assigned to a subdomain, we do not consider the aforementioned pooling (clustering) methods.
%
\begin{figure}
     \includegraphics[width=0.5\textwidth]{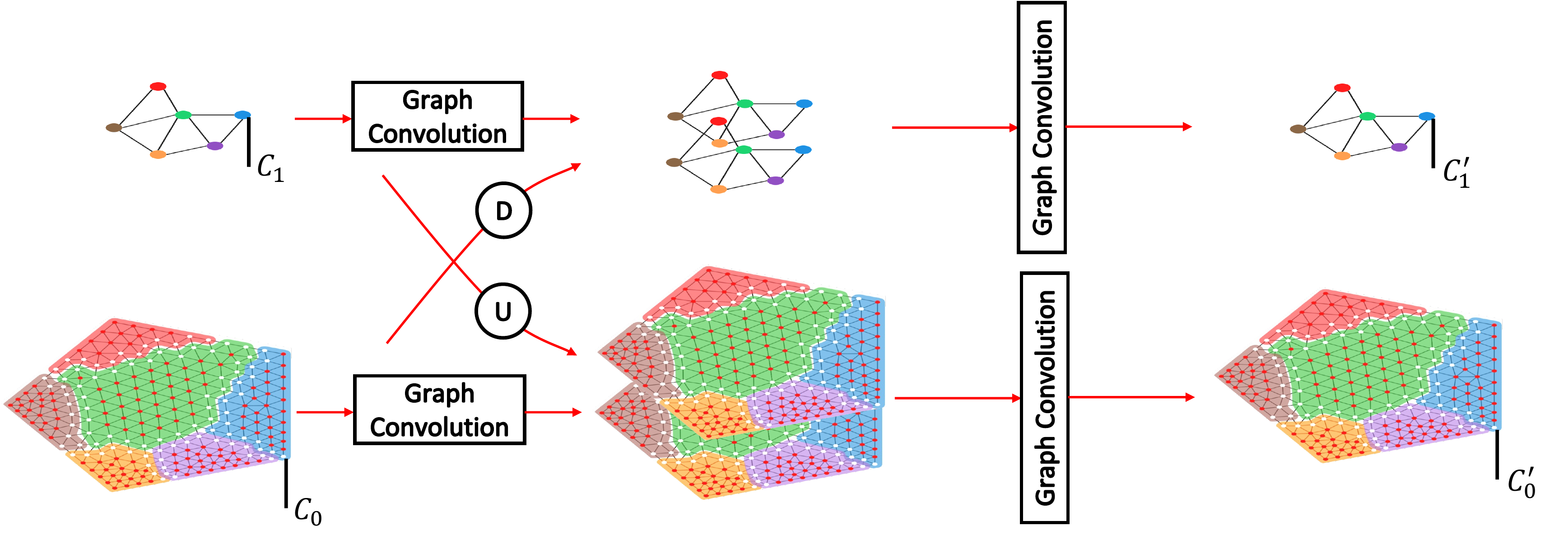}
     \vspace{-5mm}
  \caption{One layer of MG-GNN. $c_{i}$ and $c^{\prime}_{i}$ denote the feature dimensions of different levels before and after passing through an MG-GNN layer, respectively.}\label{fig:MG-GNN}\label{fig:2Lmggnn}
\end{figure}
\begin{figure}
     \includegraphics[width=0.5\textwidth]{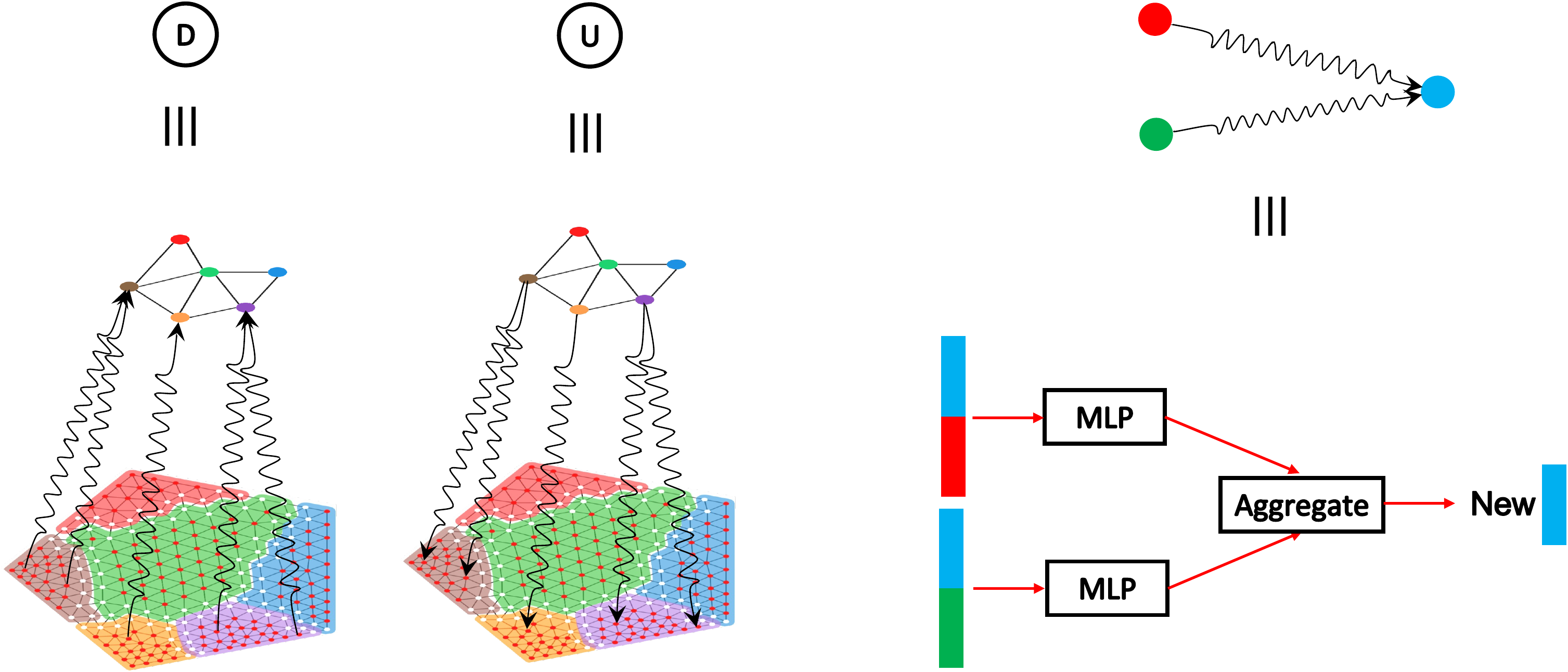}
     \vspace{-5mm}
  \caption{Upsampling and downsampling in MG-GNN.}\label{fig:UDsampling}
\end{figure}

\section{Optimization problem and loss function}\label{sec:loss}

In this section, we denote the $\ell^2$ norm of a matrix or vector by $\|\cdot\|$ and the spectral radius of matrix $T$ by $\rho(T)$. Our objective is to minimize the asymptotic convergence factor of the two-level ORAS method, defined as minimizing $\rho(T)$, where $T = I-M_{\text{2-ORAS}}A = (I-C_{\text{2-ORAS}}A)(I-M_{\text{ORAS}}A)$ is the error propagation operator of the method. Since $T$ is not necessarily symmetric, $\rho(T)$ is formally defined as the extremal eigenvalue of $T^{T}T$. As discussed in~\citet{wang2019backpropagation}, numerical unsuitability of backpropagation of an eigendecomposition makes it infeasible to directly minimize $\rho(T)$. To this end,~\citet{luz2020learning} relax the spectral radius to the Frobenius norm (which is an upper bound for it), and minimize that instead. However, for the case of optimizing one-level DDM methods, the work in~\citet{taghibakhshi2022learning} highlights that the Frobenius norm is not a ``tight'' upper bound for $\rho(T)$, and considers minimizing a relaxation of $\rho(T)$ inspired by Gelfand's formula, $\forall_{K\in\mathbb{N}}\;\;\;\rho(T) \leq \|T^K\|^{\frac{1}{K}} = \sup_{x: \|x\|=1} (\|T^{K}x\|)^\frac{1}{K}$. We present a modified version of the loss function introduced by~\citet{taghibakhshi2022learning} and, in Section~\ref{subsec:results}, we show the necessity of this modification for improving the two-level RAS results.


Consider the discretized problem with DoF set $D$ of size $n$, decomposed into $S$ subdomains, $D_1^\delta, D_2^\delta, \ldots, D_S^\delta$ with overlap $\delta$. The GNN takes $D$, its decomposition, and a sparsity pattern for the interface values and that of the interpolation operator as inputs and its outputs are the learned interface values and interpolation operator (see Appendix~\ref{sec:appendix-model} for more discussion on inputs and outputs of the network):
\begin{equation}
\label{eq:GNNinout}
{P^{(\theta)}, L_{1}^{(\theta)}, L_{2}^{(\theta)}, \ldots, L_{S}^{(\theta)}} \leftarrow \psi^{(\theta)}(D).
\end{equation}where $\psi^{\theta}$ denotes the GNN, and $\theta$ represents the learnable parameters in the GNN\@.

We obtain the modified two-level ORAS (Optimized Restricted Additive Schwarz) operator by using the learned coarse grid correction operator, $C^{\theta}_{\text{2-ORAS}} = P^{(\theta)}\left((P^{(\theta)})^{T}AP^{(\theta)}\right)^{-1}\left(P^{(\theta)}\right)^T$, and the fine grid operator, $M^{(\theta)}_{\text{ORAS}}$ from~(\ref{eq:GNNinout}).  The associated 2-level error propagation operator is then given by $T^{(\theta)} = (I - C^{(\theta)}_{\text{2-ORAS}}A)(I - M^{(\theta)}_{\text{ORAS}}A)$.

In order to obtain an approximate measure of $\rho(T^{(\theta)})$ while avoiding eigendecomposition of the error propagation matrix, similar to~\citet{taghibakhshi2022learning}, we use stochastic sampling of $\left\|\left(T^{(\theta)}\right)^K\right\|$, generated by the sample set $X\in\mathbb{R}^{n\times m}$ for some $m\in\mathbb{N}$, given as
\begin{equation}
X = [x_1, x_2, \ldots, x_m], \forall_{j} \; x_j\sim \mathbb{R}^{n} \; \text{uniformly}, \|x_j\| = 1,
\end{equation}
where each $x_j$ is sampled uniformly randomly on a unit sphere in $\mathbb{R}^{n}$ using the method introduced in~\citet{box1958note}. We then define
\begin{multline}
Y^{(\theta)}_{K} = \\ \left\{\left\|\left( T^{(\theta)}\right)^{K}x_1\right\|, \left\|\left(T^{(\theta)}\right)^{K}x_2\right\|, \ldots, \left\|\left(T^{(\theta)}\right)^{K}x_m\right\|\right\}.
\end{multline}

Note that $\left\|\left( T^{(\theta)}\right)^{K}x_j\right\|$ is a lower bound for $\left\|\left( T^{(\theta)}\right)^{K}\right\|$. \citet{taghibakhshi2022learning} use $\mathcal{L}^{(\theta)} = \max(Y_K^{(\theta)})$ as a practical loss function. However, for large values of $K$, this loss function suffers from vanishing gradients. Moreover, as we show in Section~\ref{subsec:results}, employing this loss function results in inferior performance of the learned method in comparison to two-level RAS\@. To overcome these issues, we define $Z_{k}^{(\theta)} = \max((Y^{(\theta)}_{k})^{\frac{1}{k}})$ for $ 1\ \leq k \leq K$ to arrive at a new loss function,
\begin{equation}
\mathcal{L}^{(\theta)} = \langle\text{softmax}(Z^{(\theta)}),\; Z^{(\theta)}\rangle+\gamma\text{tr}\left((P^{(\theta)})^{T}AP^{(\theta)}\right)\label{eq:new_loss},
\end{equation}
where $Z^{(\theta)} = \left(Z_1^{(\theta)}, Z_2^{(\theta)}, \ldots, Z_K^{(\theta)}\right)$, $0<\gamma$ is an adjustable constant, and $\text{tr(M)}$ is the trace of matrix $M$. Adding the term $\text{tr}((P^{(\theta)})^{T}AP^{(\theta)})$ is inspired by energy minimization principles, to obtain optimal interpolation operators in theoretical analysis of multilevel solvers~\cite{xu1992iterative, wan1999energy}. In Section~\ref{subsec:results}, we show the significance of this term in the overall performance of our model. Nevertheless, for the first part of the new loss function~(\ref{eq:new_loss}), we prove that it convergence to the spectral radius of the error propagation matrix in a suitable limit. First, we include two lemmas:
\begin{lemma}\label{lem:useful}
For $x,y\in\mathbb{R}$, with $0\le y \le x$ and any $K\in\mathbb{N}$, $x^{\frac{1}{K}}-y^{\frac{1}{K}}\le(x-y)^{\frac{1}{K}}$
\end{lemma}
\begin{proof}
See Lemma 3 from~\citet{taghibakhshi2022learning}.
\end{proof}

\begin{lemma}\label{lem:lem}

For any nonzero square matrix $T\in\mathbb{R}^{n\times n}$, $k\in\mathbb{N}$, $\epsilon, \xi>0$, and $0<\delta<1$, there exists $M\in\mathbb{N}$ such that for any $m\ge M$, if we choose $x_1, x_2, \ldots, x_m$ uniformly random from $\{x\in\mathbb{R}^{n}\,|\, \|x\| = 1\}$, and $Z = \max\{\|T^{k}x_1\|^{\frac{1}{k}}, \|T^{k}x_2\|^{\frac{1}{k}}, \ldots, \|T^{k}x_m\|^{\frac{1}{k}}\}$ then, with a probability of at least ($1-\delta$), the following hold:
\begin{align}
&0\le\|T^{k}\|^{\frac{1}{k}}-Z\le\epsilon,\label{eq:part1}\\
&\rho(T)-\xi\le Z\label{eq:part2}.
\end{align}
\end{lemma}
\begin{proof}
The left side of the first inequality is achieved by considering the definition of matrix norm, i.e.\ for any $1\le i\le m$, $\|T^{k}x_i\|\le\sup\limits_{\|x\|=1}\|T^{k}x\| = \|T^{k}\|$, then taking the $k^{\text{th}}$ root of both sides. For the right side of the first inequality, consider the point $x^{*}\in\{x\in\mathbb{R}^{n}\,|\, ||x|| = 1\}$ such that $ \|T^{k}x^{*}\|=\sup\limits_{\|x\|=1}\|T^{k}x\|$ (such a point exists since $\mathbb{R}^{n}$ is finite dimensional). Let $S$ be the total volume of the surface of an $n$ dimensional unit sphere around the origin, and denote by $\tilde{S}$ the volume on this surface within distance $\tilde{\epsilon}$ of the point $x^{*}$ in the $\ell^2$ measure, for $\tilde{\epsilon} = \frac{\epsilon}{\|T^k\|^{\frac{1}{k}}}$. Let $m \geq M_1 > \frac{\log(\delta)}{\log\left(1-\frac{\tilde{S}}{S}\right)}$, then, since $0<\delta<1$, we have:
\begin{equation}
P(\|x^{*} - x_{i}\|>\tilde{\epsilon},\;\;\forall_{i})=\left(1-\frac{\tilde{S}}{S}\right)^m\le\delta
\end{equation}Therefore, with probability of at least $(1-\delta)$, there is one $x_i$ within the $\tilde{\epsilon}$ neighborhood of $x^*$ on the unit sphere. Without loss of generality, let $x_1$ be that point. Using Lemma~\ref{lem:useful} and the reverse triangle inequality, we have
\begin{multline}
  \left\|T^k x^*\right\|^\frac{1}{k}-\left\|T^k x_1\right\|^\frac{1}{k}\leq\left(\left\|T^k x^*\right\|-\left\|T^k x_1\right\|\right)^{\frac{1}{k}} \\
  \le\left\|T^k\right\|^{\frac{1}{k}}\|x^*-x_1\|^{\frac{1}{k}}\le\|T^k\|^{\frac{1}{k}}\tilde{\epsilon} = \epsilon
\end{multline}
which finishes the proof for the right side of the first inequality.

For the second inequality, since $\rho(T)\le\|T^{k}\|^{\frac{1}{k}}$, choose $M_2$ such that, with probability $1-\delta$, \eqref{eq:part1} holds for $\epsilon = \|T^{k}\|^{\frac{1}{k}} - \rho(T)+\xi > 0$.  Rearranging~\eqref{eq:part1} then yields~\eqref{eq:part2} for any $m \geq M = \max\{M_1,M_2\}$.
\end{proof}

We next state the main result on optimality.
\begin{theorem}\label{thm:optimality}
For any nonzero matrix $T$, $\epsilon>0$, and $\delta <1$, there exist $M, K\in\mathbb{N}$ such that for any $m>M$, if one chooses $m$ points, $x_j$, uniformly at random from $\{x\in\mathbb{R}^n\,|\, \|x\| = 1\}$ and defines $Z_k = \max\{\|T^{k}x_1\|^{\frac{1}{k}}, \|T^{k}x_2\|^{\frac{1}{k}}, \ldots, \|T^{k}x_m\|^{\frac{1}{k}}\}$, then $Z=(Z_1, Z_2, \ldots, Z_K)$ satisfies:
  \begin{equation}\label{eq:opt}
  P\left(\left|\langle\text{softmax}(Z),\; Z\rangle-\rho(T)\right|\le\epsilon\right) > 1 - \delta.
  \end{equation}
  \end{theorem}
\begin{proof}
Since $\rho(T)\le\|T^{k}\|^{\frac{1}{k}}$ for any $k$ and $\lim_{k\rightarrow\infty} \|T^k\|^{\frac{1}{k}} = \rho(T)$, for any $0<\alpha$, there exists $K^*\in\mathbb{N}$ such that for any $k>K^*$, $0\le\|T^{k}\|^{\frac{1}{k}}-\rho(T)<\alpha$. Take $0<\alpha<\min\{\frac{\epsilon}{2}, \log(\frac{e^{-\epsilon}(\epsilon+\rho(T))}{\frac{\epsilon}{2}+\rho(T)})\}$, let $u = \max\{\max\limits_{1\le k\le K^*}\{\|T^k\|^{\frac{1}{k}}\}+\alpha, \rho(T)+2\alpha\}$, $\tilde{\delta}=1-(1-\delta)^{\frac{1}{K}}$, and take $K > \max\{\frac{K^*(ue^u-(\rho(T)+\alpha)e^{\rho(T)+\alpha})}{e^{\rho(T)-\epsilon}(\epsilon+\rho(T)-(\rho(T)+\alpha)e^{\alpha+\epsilon})}, K^*\}$. Note that, by the choice of $\alpha$, we have $\rho(T)+\alpha<\rho(T)+\frac{\epsilon}{2}$ and $e^{\alpha+\epsilon}<\frac{\rho(T)+\epsilon}{\rho(T)+\frac{\epsilon}{2}}$, which (along with the choice of $u$) guarantees a positive $K$. By Lemma~\ref{lem:lem}, for any $1\le i\le K^*$ and $K^{*}<j\le K$, there exists $n_i, n_j\in\mathbb{N}$ such that:
\begin{align}
&P(\rho(T)-\epsilon\le Z_i\le u)>1-\tilde{\delta}&\text{for}\;\; m>n_i,\label{eq:less-k}\\
&P(\rho(T)-\epsilon\le Z_j\le\rho(T)+\alpha)>1-\tilde{\delta}&\text{for}\;\; m>n_j\label{eq:great-k}.
\end{align}For any $1\le k\le K$, take $n_k$ independent points on unit sphere so that the above inequalities are satisfied for all $k$. Note that this can be achieved by taking $M=\sum\limits_{k=1}^{K}n_k$. Since the points for satisfying equations~(\ref{eq:less-k}) and~(\ref{eq:great-k}) are chosen independently, for any $m>M$, with probability of at least $\prod\limits_{k=1}^{K}(1-\tilde{\delta}) = 1- \delta$, we have $\rho(T)-\epsilon\le Z_k$ for all $1\le k\le K$. Consequently:
\begin{align}
&-\epsilon=\frac{(\rho(T)-\epsilon)\sum\limits_{i=1}^{K}e^{Z_i}}{\sum\limits_{i=1}^{K}e^{Z_i}}-\rho(T)\le \frac{\sum\limits_{i=1}^{K}Z_i e^{Z_i}}{\sum\limits_{i=1}^{K}e^{Z_i}}-\rho(T)\\
&=\langle\text{softmax}(Z),\; Z\rangle-\rho(T)\\
&\le \frac{uK^*e^u+(K-K^*)(\rho(T)+\alpha)e^{\rho(T)+\alpha}}{Ke^{\rho(T)-\epsilon}}-\rho(T)\\
&= \frac{K^*(ue^u-(\rho(T)+\alpha)e^{\rho(T)+\alpha})}{Ke^{\rho(T)-\epsilon}}\nonumber\\
&\hspace{7em} +(\rho(T)+\alpha)e^{\alpha+\epsilon} - \rho(T) \le\epsilon,
\end{align}
where the last inequality is obtained by the choice of $K$.
\end{proof}

In addition to these properties of the loss function, we now show that obtaining the learned parameters using our MG-GNN architecture scales linearly with the problem size.
\begin{theorem}\label{thm:complexity}
The time complexity to obtain the optimized interface values and interpolation operator using our MG-GNN is $O(n)$, where $n$ is the number of nodes in the grid.
\end{theorem}
\begin{proof}
Every in/cross-level graph convolution of the MG-GNN has linear complexity. This must be the case when the graph convolution is a message passing scheme due to the sparsity in finite-element triangulations. For the case that the graph convolution is a TAGConv layer, we have $y = \sum_{\ell=1}^{L}G_{\ell}x_{\ell} + b\mathbf{1}_{n}$, where $x_{\ell}\in\mathbb{R}^{n}$ are the node features, $L$ is the node feature dimension, $b$ is a learnable bias, and $G_{\ell} \in  \mathbb{R}^{n\times n}$ is the graph filter. In TAGConv layers, the graph filter is given as $G_{\ell} = \sum_{j=0}^{J}g_{\ell,j}M^{j}$, where $M$ is the adjacency matrix, $J$ is a constant, and $g_{\ell,j}$ are the filter polynomial coefficients. In other words, the graph filter it is a polynomial in the adjacency matrix $M$ of the graph. Moreover, the matrix $M$ is sparse, hence obtaining $M^j$ has $O(n)$ computation cost, resulting in full TAGConv $O(n)$ time complexity. Moreover, for both the interface value head and the interpolation head of the network, the cost of calculating edge feature and the feature networks are $O(n)$, resulting in overall $O(n)$ cost of MG-GNN\@.
\end{proof}


\section{Experiments}\label{sec:experiments}

\subsection{Training}\label{subsec:training}

We train each model on 1000 grids of sizes ranging from 800--1000 nodes. The grids are generated randomly as a convex polygon and using PyGMSH~\citep{Schlmer_pygmsh_A_Python} for meshing its interior. The subdomains are generated using Lloyd clustering on the graph~\citep{bell2008algebraic}, the subdomain overlap is set to one, and the weights of the edges along the boundary determine the interface value operators, $L^{(\theta)}_i$. As shown in the interpolation head of the network in Appendix~\ref{sec:appendix-model} Figure~\ref{fig:overallgnn}, the weight of the edges connecting the coarse and fine grids determine the interpolation operator.  In our case, the edges between the coarse and fine grids connect every fine node to the coarse node corresponding to its own subdomain and its neighboring subdomains.  Alternatively, every fine node could connect only to the coarse node corresponding to its subdomain but, as we discuss in Section~\ref{subsec:results}, this significantly impacts the performance of the model. Moreover, each row of the interpolation operator, $P^{(\theta)}$, is scaled to have sum of one, as would be the case for classical interpolation operators. Figure~\ref{fig:training_example} shows several example training grids.

The model is trained for 20 epochs with batch size of 10 using the ADAM optimizer~\cite{kingma2014adam} with a fixed learning rate of $5\times 10^{-4}$. For the full discussion on model architecture, see Appendix~\ref{sec:appendix-model} and Figure~\ref{fig:overallgnn}. For the loss function parameters introduced in Section~\ref{sec:loss}, we use $K=10$ iterations and $m=100$ samples. We developed our code\footnote[1]{All code and data for this paper is at \url{https://github.com/JRD971000/Code-Multilevel-MLORAS/} (MIT licensed).} using PyAMG~\cite{BeOlSc2022}, NetworkX~\cite{hagberg2008exploring}, and PyTorch Geometric~\cite{Fey_Lenssen_2019}. All training is executed on an i9 Macbook Pro CPU with 8 cores.
In the training procedure, we aim to minimize the convergence of the stationary algorithm and, as described in Section~\ref{sec:loss}, we develop a loss function to achieve this goal by numerically minimizing the spectral radius of the error propagation matrix. In practice, optimized RAS methods are often used as preconditioners for Krylov methods such as FGMRES\@; as shown in Appendix~\ref{sec:appendix}, the trained models using this procedure also outperform other baselines when used as preconditioners for FGMRES\@. Directly training to minimize FGMRES iterations would require using FGMRES in the training loop and backpropagation through sparse-sparse matrix multiplication~\cite{nytko2022optimized}, which is left for future studies.

\begin{figure}
     \includegraphics[width=0.48\textwidth,trim=0 10 0 10]{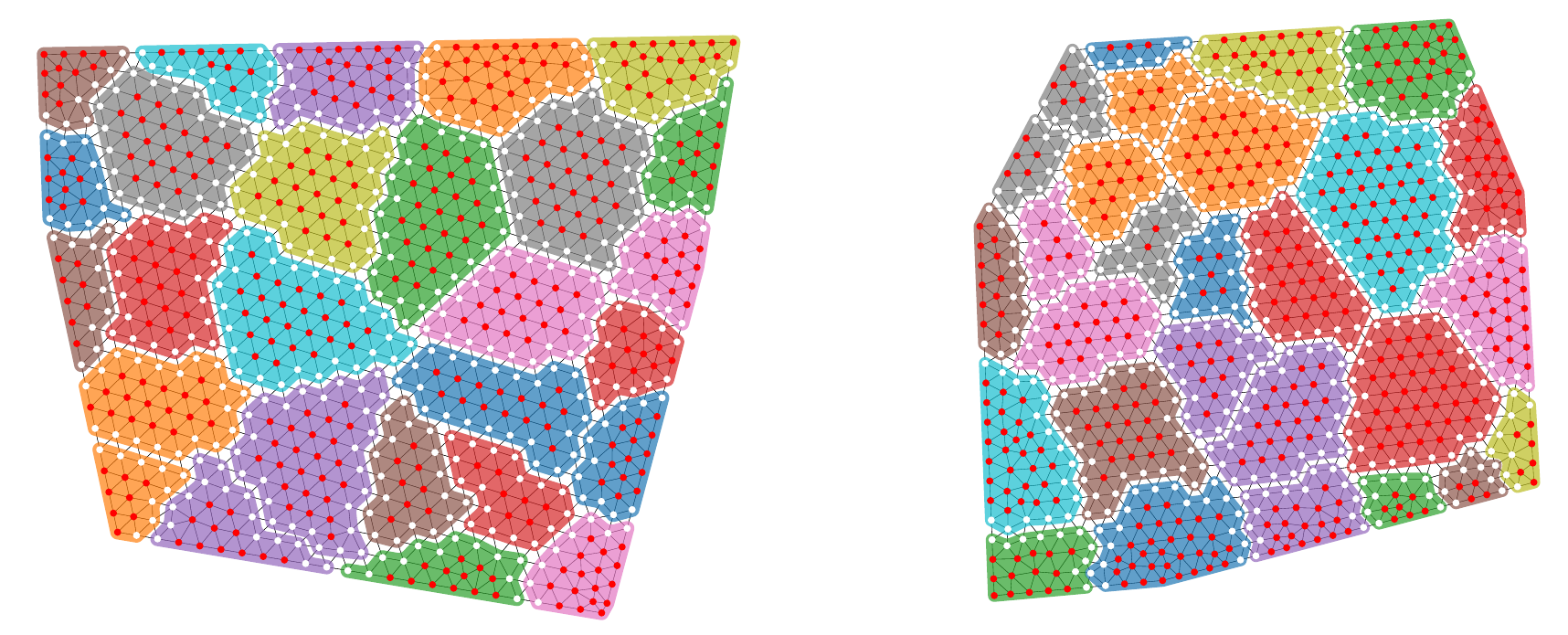}
     \vspace{-5mm}
       \caption{Training grid examples with about 1k nodes.}\label{fig:training_example}
\end{figure}

We evaluate the model on test grids that are generated in the same fashion as the training grids, but are larger in size, ranging from 800 to 60k DoFs. An example of a test grid is shown in Figure~\ref{fig:test_grids}.

\begin{figure}
     \includegraphics[width=0.48\textwidth,trim=0 5 0 5]{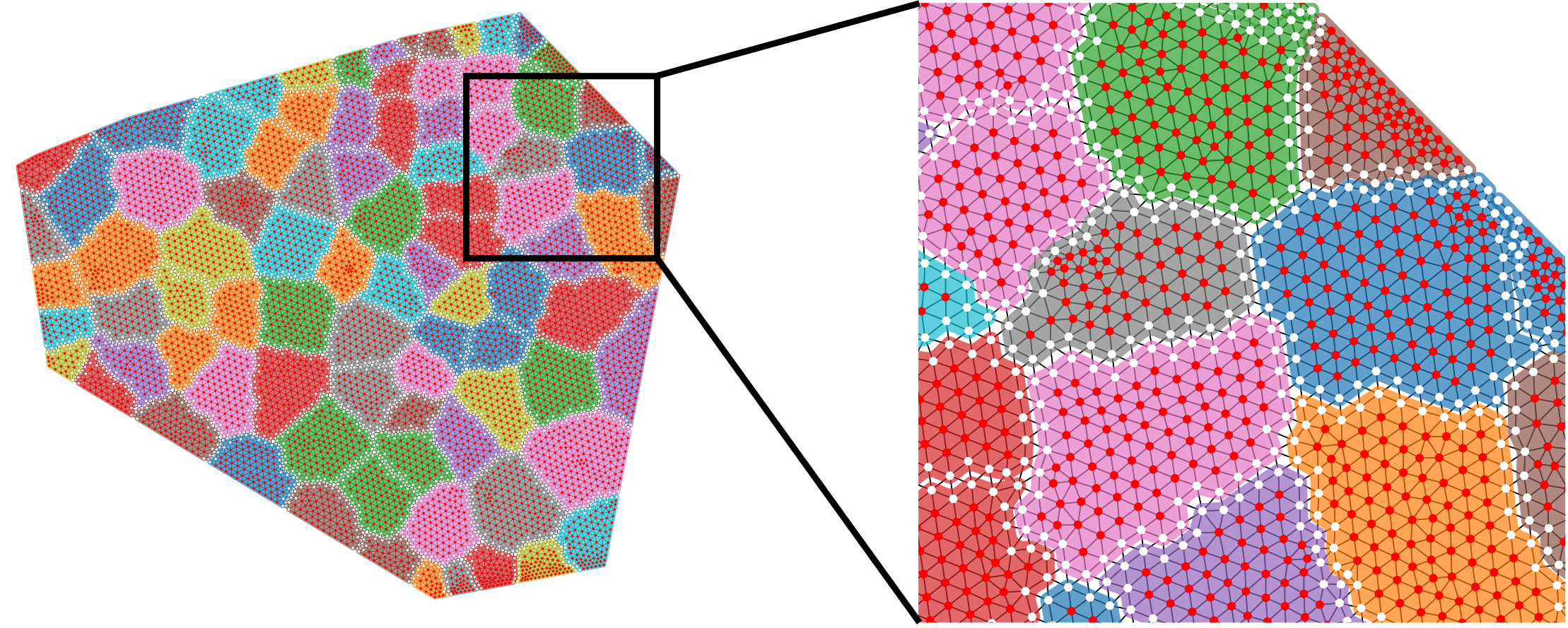}
     \vspace{-5mm}
       \caption{Test grid example with about 7.4k nodes.}\label{fig:test_grids}
\end{figure}

\subsection{Interface values and interpolation operator}\label{subsec:learningLP}

As mentioned in the Section~\ref{sec:loss}, to optimize two-level RAS, one could optimize the parameters in the interface conditions~(\ref{eq:interface}) and/or the interpolation operator~(\ref{eq:cgsolve}). For one-level RAS, on the other hand, there is no interpolation operator (since there is no coarse grid), leaving only the interface values to optimize, as was explored in~\citet{taghibakhshi2022learning} and~\citet{st2007optimized}.
To compare the importance of these two ingredients in the two-level RAS optimization, we compare three different models.
Each of these models is trained as described in Section~\ref{subsec:training}; however, one of the models (labeled ``interface'') is trained by only learning the interface values (ignoring the interpolation head of the network), and using classical RAS interpolation to construct $T^{(\theta)}$. Another model, which we label ``interpolation'', only learns the interpolation operator weights, and uses zeros for interface matrices $L_{i}^{(\theta)}$ to construct $T^{(\theta)}$. The other model uses both training heads (see Figure~\ref{fig:overallgnn}), learning the interface values and the interpolation operator. We compare the performance of these models with classical RAS in Figure~\ref{fig:what2learn-stationary} as a stationary algorithm, and in Figure~\ref{fig:what2learn-loss-fgmres} as a preconditioner for a Krylov method, FGMRES\@.

\begin{figure}
   \includegraphics[width=0.48\textwidth,,trim=0 10 0 10]{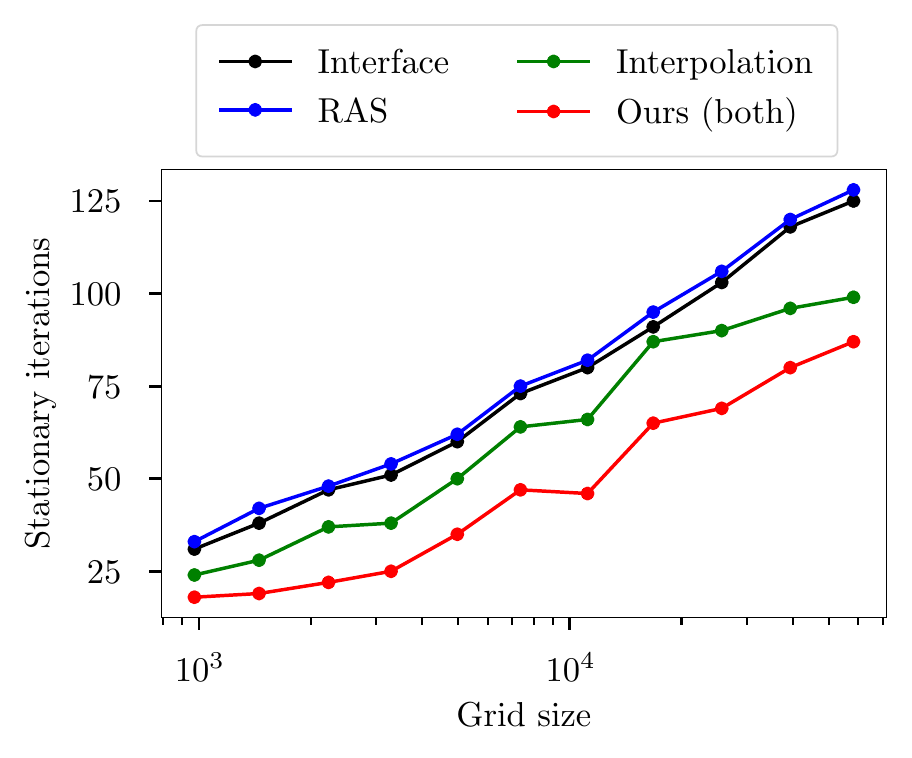}
   \vspace{-5mm}
   \caption{Effect of learning interface values, interpolation operator, or both on stationary iterations.}\label{fig:what2learn-stationary}
\end{figure}

The results show that learning the interpolation operator is more important in optimizing the 2-level RAS\@. Intuitively, the coarse-grid correction process in~\eqref{eq:cgsolve} plays an important role in scaling performance to large problems, due to its global coupling of the discrete DOFs.  The interpolation operator is critical in achieving effective coarse-grid correction.
On the other hand, the interface values are \textit{local} modifications to the subdomains (see~\eqref{eq:interface}), that cannot (by themselves) make up for a poor coarse-grid correction process.  Learning both operators clearly results in the best performance in Figures~\ref{fig:what2learn-stationary} and~\ref{fig:what2learn-loss-fgmres}, where the interpolation operator can be adapted to best complement the effects of the learned interface values.

\subsection{Loss function and sparsity variants}\label{subsec:results}

We first compare five variants of our method with the RAS baseline, as shown in Figures~\ref{fig:loss-ablation-stationary} and~\ref{fig:what2learn-loss-fgmres}. The main model is trained as described in Section~\ref{subsec:training} with the loss function from Section~\ref{sec:loss}. All but one of the variants only differ in their loss function, and share the rest of the details.  The variant labeled ``Max loss'' is trained with the loss function from~\citet{taghibakhshi2022learning}.
The variant labeled ``Max+Trace loss'' is trained with the loss function from~\citet{taghibakhshi2022learning} plus the $\gamma\text{tr}(P^T A P)$ term.  Similarly, the variant labeled ``Softmax loss'' is trained by removing the $\gamma\text{tr}(P^T A P)$ part from the loss function in~(\ref{eq:new_loss}). For the last variant, we restrict the sparsity of the interpolation operator to that obtained by only connecting every fine node to its corresponding coarse node, labelled ``DDM standard sparsity'', and trained using the loss function from~(\ref{eq:new_loss}). As shown in Figure~\ref{fig:loss-ablation-stationary}, the learned operator using this variant achieves worse performance than the baseline RAS\@.  This is partly because, for this variant, the constraint on unit row sums of the interpolation operator effectively removes most of the learned values, since many rows of interpolation have only one nonzero entry in this sparsity pattern.
\begin{figure}
  \centering
   \includegraphics[width=0.48\textwidth,trim=0 10 0 10]{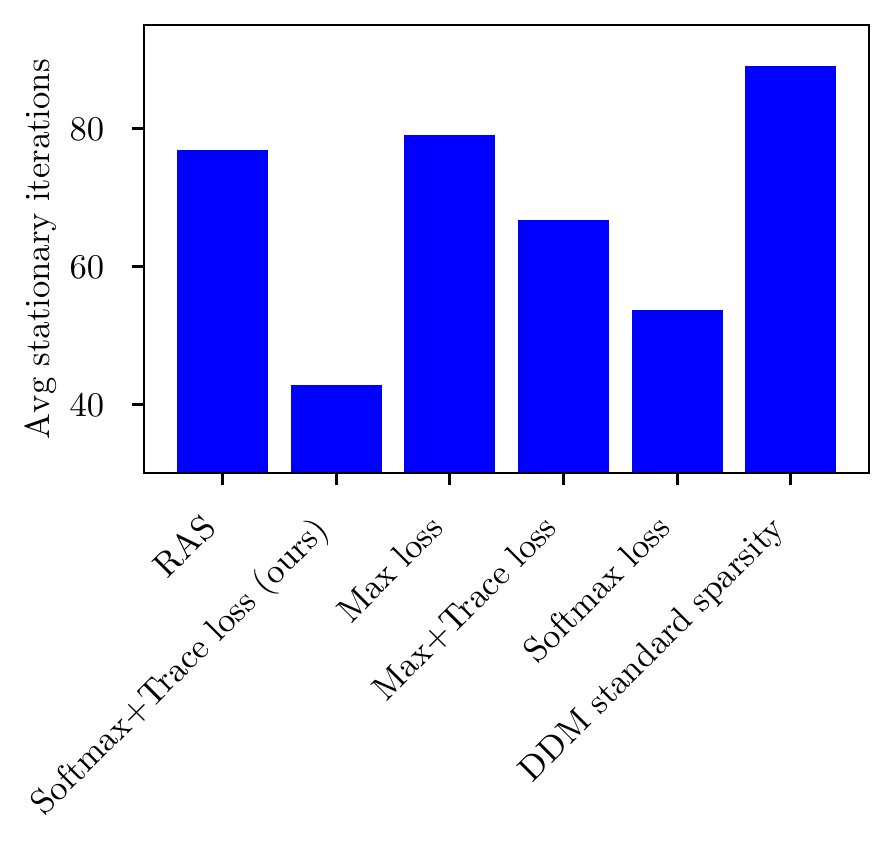}
   \vspace{-5mm}
   \caption{Effect of every ingredient in the model on average stationary iterations. All three variants outperforming the RAS baseline are utilizing modifications introduced in this paper (see~(\ref{eq:new_loss})) compared to the ``Max loss'' from~\citet{taghibakhshi2022learning}.}\label{fig:loss-ablation-stationary}
\end{figure}

To show the effectiveness of the coarse-grid correction and the learned operator, we also compare two-level RAS and our two-level learned RAS (MLORAS 2-level) with one-level RAS and one-level optimized RAS from~\citet{taghibakhshi2022learning} in Figures~\ref{fig:level-ablation-stationary} and~\ref{fig:level-ablation-fgmres}.

\begin{figure}
  \centering
     \includegraphics[width=0.48\textwidth,trim=0 10 0 10]{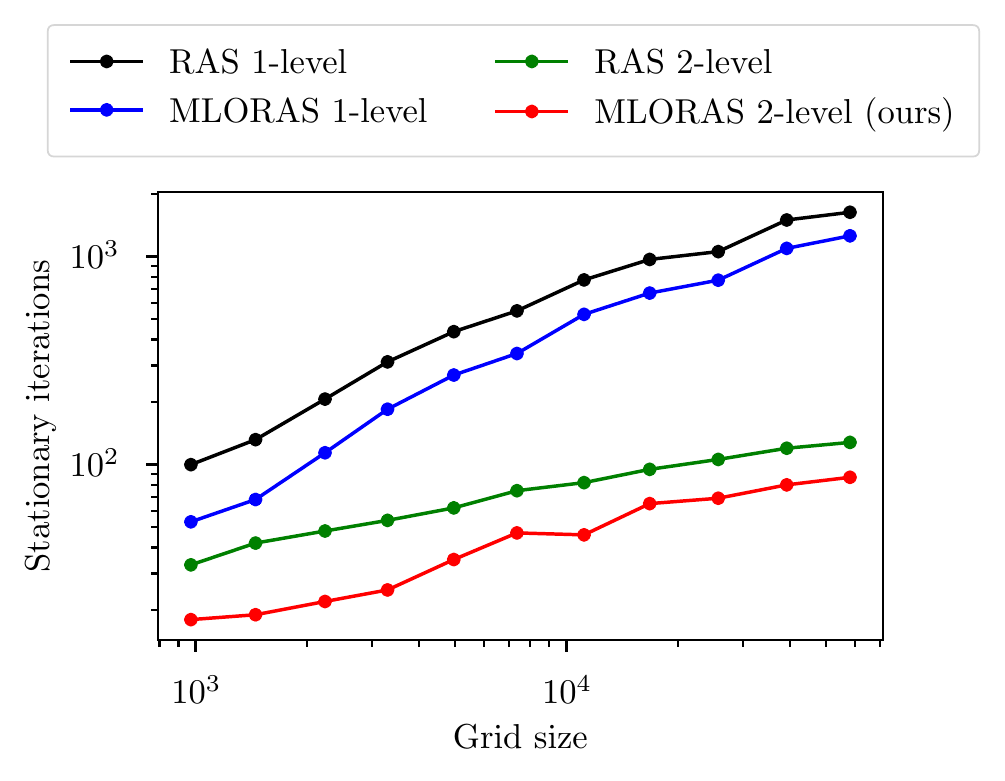}
     \vspace{-5mm}
     \caption{Comparison of stationary iterations of 2-level methods with 1-level methods from~\citet{taghibakhshi2022learning}.}\label{fig:level-ablation-stationary}
\end{figure}

\subsection{Comparison to Graph U-net and number of layers}

In this section, the performance of Graph U-net and MG-GNN with different numbers of layers is studied. Figures~\ref{fig:ablation-stationary} and~\ref{fig:level-ablation-fgmres} show the performance of each of the models as stationary iterations and preconditioners for FGMRES, respectively. For a fair comparison, the MG-GNN and graph U-nets that share the same number of layers also have the same number of trainable parameters. As shown here, the best performance is achieved with 4 layers of MG-GNN, and MG-GNN strictly outperforms the graph U-net architecture with the same number of layers.

\section{Conclusion}
In this study, we proposed a novel graph neural network architecture, which we call multigrid graph neural network (MG-GNN), to  learn two-level optimized restricted additive Schwarz (optimized RAS or ORAS) preconditioners. This new MG-GNN ensures cross-scale information sharing at every layer, eliminating the need to use multiple graph convolutions for long range information passing, which was a shortcoming of prior graph network architectures. Moreover, MG-GNN scales linearly with problem size, enabling its use for large graph problems. We also introduce a novel unsupervised loss function, which is essential to obtain improved results compared to classical two-level RAS\@. We train our method using relatively small graphs, but we test it on graphs which are orders of magnitude larger than the training set, and we show our method consistently outperforms the classical approach, both as a stationary algorithm and as an FGMRES preconditioner.
\begin{figure}
  \centering
     \includegraphics[width=0.48\textwidth,trim=0 10 0 10]{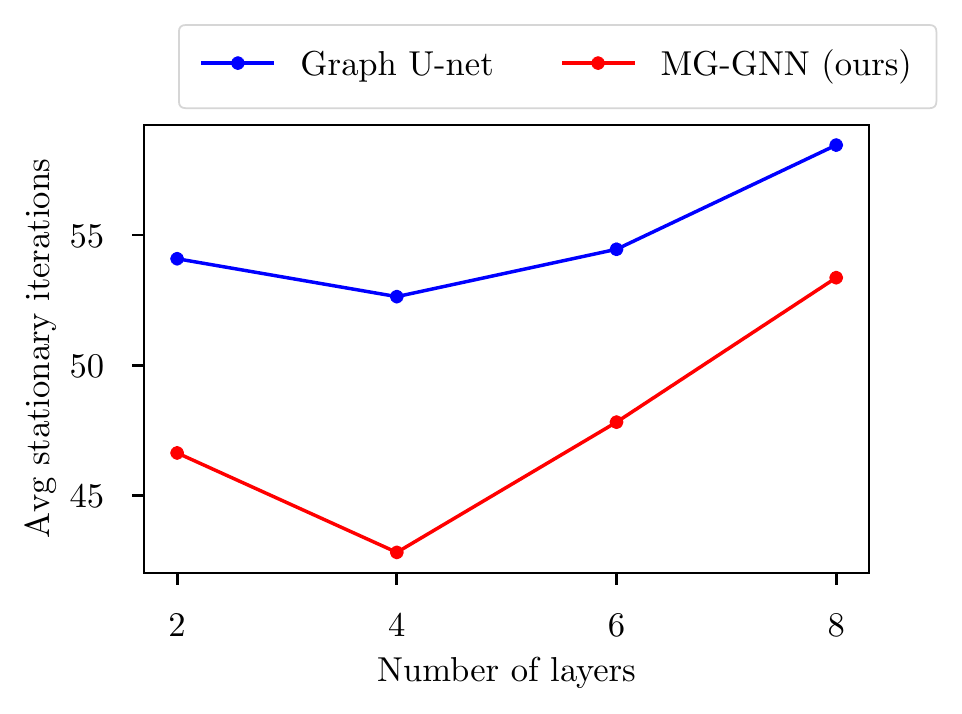}
     \vspace{-5mm}
       \caption{Average stationary iterations of graph U-net and MG-GNN with different number of layers on the test set.}\label{fig:ablation-stationary}
\end{figure}


\bibliography{ref-multilevel-mloras}
\bibliographystyle{icml2023}

\newpage
\appendix
\onecolumn
\section{FGMRES plots}\label{sec:appendix}

In Section~\ref{sec:experiments}, in Figures~\ref{fig:what2learn-stationary} to Figure~\ref{fig:ablation-stationary}, the performance of the methods was evaluated by considering the convergence of stationary iterations. Here, we present another possible evaluation criterion, assessing the number of iterations to convergence for the preconditioned systems using FGMRES, a standard Krylov method. The following figures are analogous to those provided in the main paper, and demonstrate that our method also achieves superior results compared to other methods, and that the MG-GNN architecture outperforms graph U-nets.

\begin{figure}[H]
  \centering
   \includegraphics[width=0.48\textwidth]{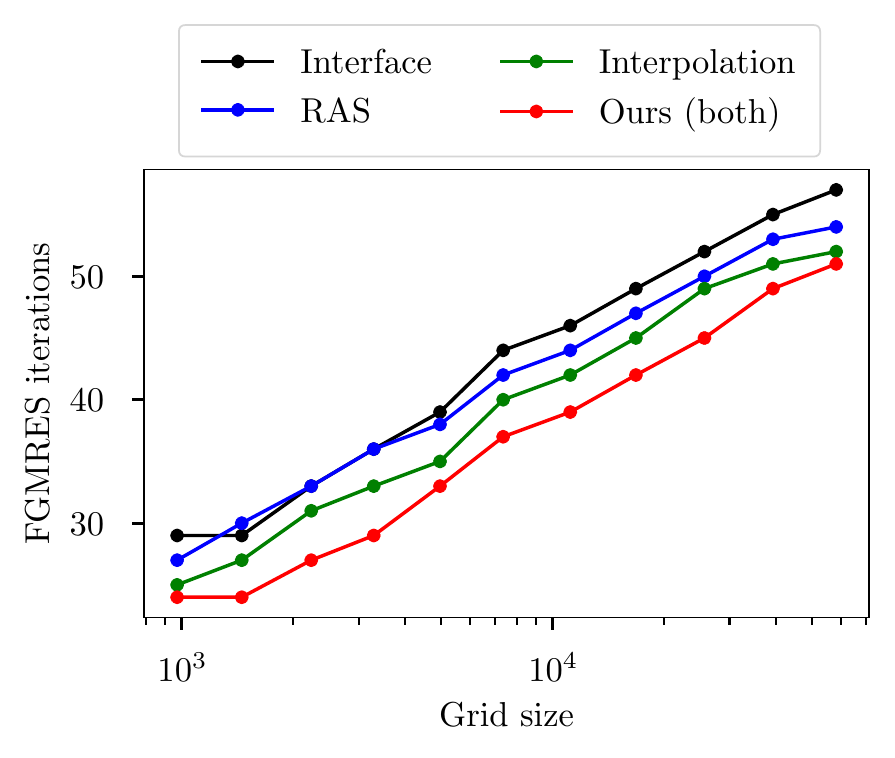}
   \hfill
   \includegraphics[width=0.48\textwidth]{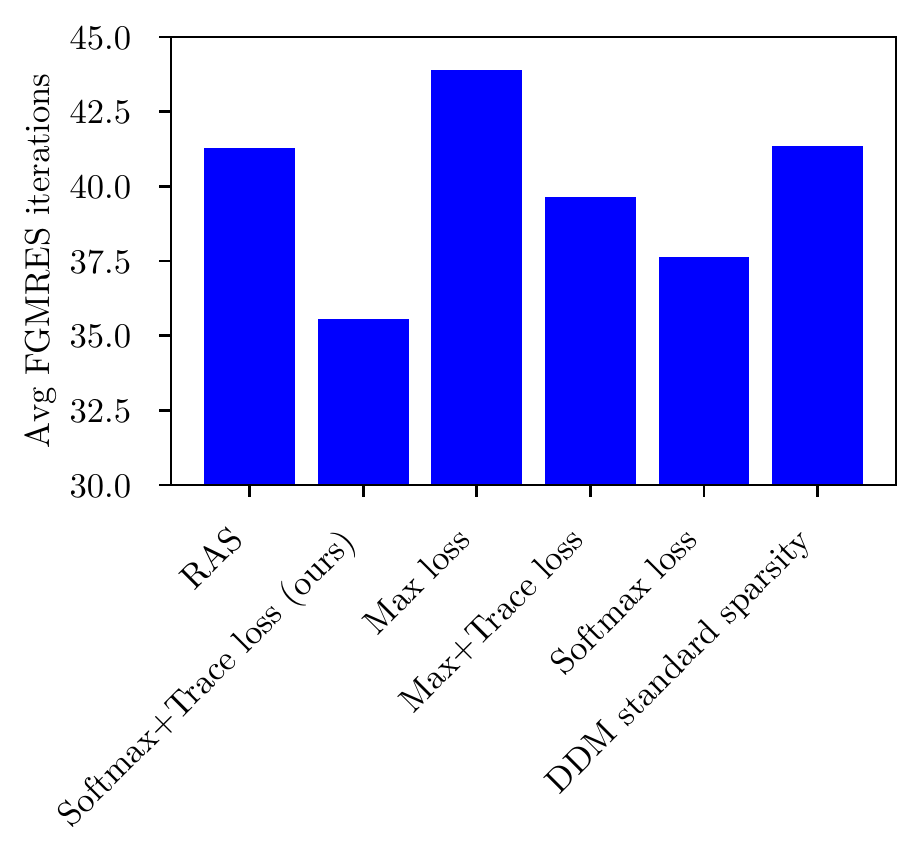}
  \caption{Left: Effect of learning interface values, interpolation operator, or both on FGMRES iterations. All three variants outperforming the RAS baseline are utilizing modifications introduced in this paper~(\ref{eq:new_loss}) compared to the ``Max loss'' from~\citet{taghibakhshi2022learning}. Right: Effect of every ingredient in the model on average FGMRES iterations.}\label{fig:what2learn-loss-fgmres}
\end{figure}

\begin{figure}[H]
  \centering
     \includegraphics[width=0.48\textwidth]{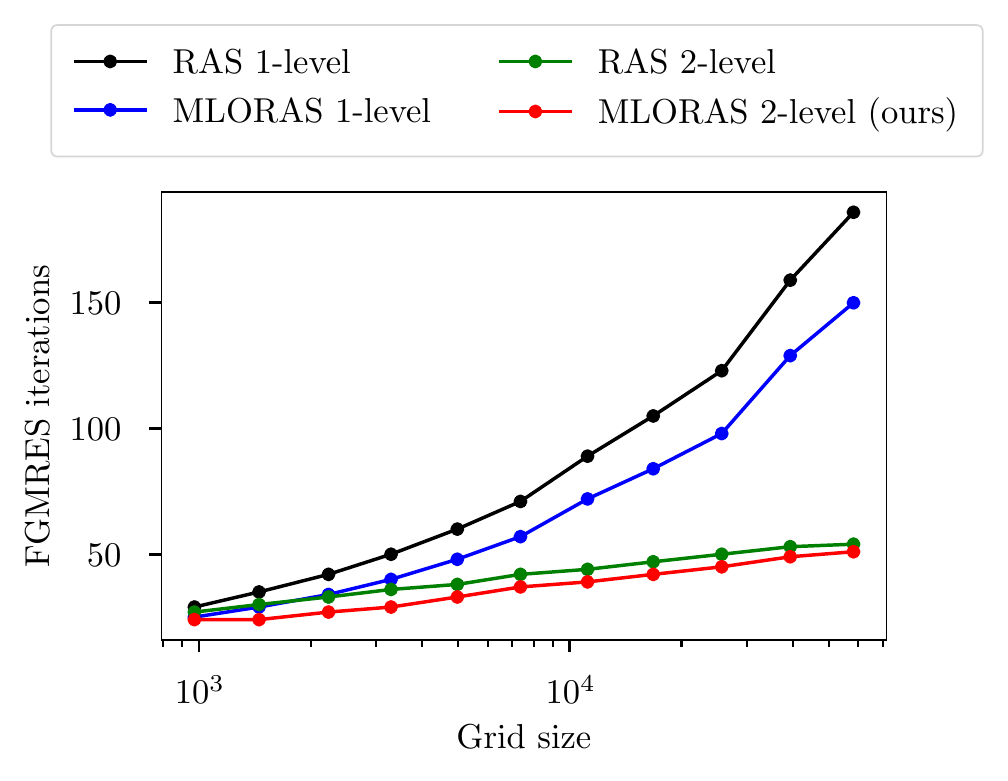}
     \hfill
     \includegraphics[width=0.48\textwidth]{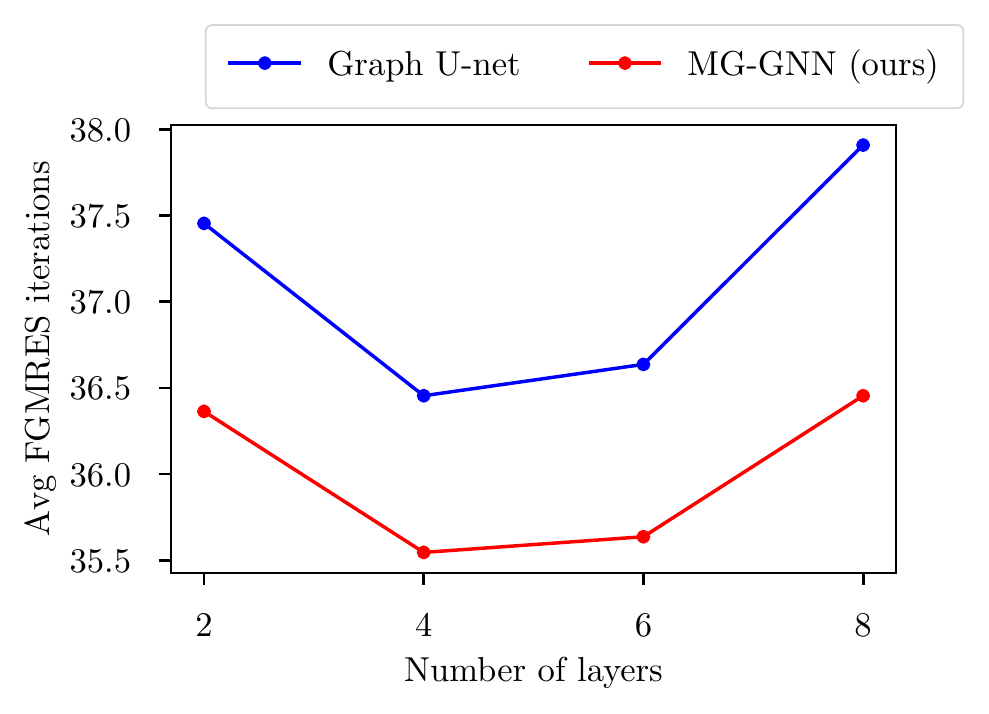}
      \caption{Left: Comparison of FGMRES iterations of 2-level methods with 1-level methods from~\citet{taghibakhshi2022learning}. Right: Average FGMRES iterations of graph U-net and MG-GNN with different number of layers on the test set.}\label{fig:level-ablation-fgmres}
\end{figure}

\section{Model architecture}\label{sec:appendix-model}

\paragraph{Inputs and outputs:} The model takes any unstructured grid as its input, which consists of the node features, edge features, and adjacency matrix of both the fine and coarse grids. Every node on the fine level has a binary feature, indicating whether it lies on the boundary of a subdomain. Fine level edge features are obtained from the discretization of the underlying PDE, $A$, and the adjacency matrix of the fine level is simply the sparsity of $A$. Similar attributes for the coarse level are obtained as described in Section~\ref{sec:mg-gnn}, Equations~(\ref{eq:next_feature}) and~(\ref{eq:next_A}), and Lloyd aggregation has been used for obtaining subdomains throughout. The outputs of the model are the learned interface values and the interpolation operator.

We use node and edge preprocessing (3 fully connected layers of dimension 128, followed by ReLU activations, in the node and feature space, respectively) followed by 4 layers of MG-GNN\@. For $\text{GNN}^{(\ell)}$ in~(\ref{eq:gnni}) and $F^{\ell\rightarrow \ell}$ in~(\ref{eq:fij}), we use a TAGConv layer~\cite{du2017topology} and, for $F^{\ell\rightarrow k}$ with $\ell\ne k$, we use a heterogeneous message passing GNN as shown in Equation~(\ref{eq:mpnn1}). Specifically, we choose summation as the permutation invariant operator in~(\ref{eq:mpnn1}) and, for the MLPs, we use two fully connected layers of size 128 with ReLU nonlinearity for $f^{\ell\rightarrow k}$ and $g^{\ell\rightarrow k}(x,y) = x$.

Following the MG-GNN layers, the network will split in two heads, each having a stack layer (which essentially concatenates the features of nodes on each side of every edge) and an edge feature post-processing (see Figure~\ref{fig:Edgepostprocessing} for details). The edge weights between the coarse and fine level are the learned  interpolation operator weights, and the edge values along the subdomains in the fine level are the learned interface values. The upper head of the network has a masking block at the end, which masks the edge values that are not along the boundary, hence only outputting the learned interface values. The overall GNN architecture for learning the interpolation operator and the interface values is shown in Figure~\ref{fig:overallgnn}.

\begin{figure}[H]
  \centering
  \includegraphics[width=1.0\textwidth]{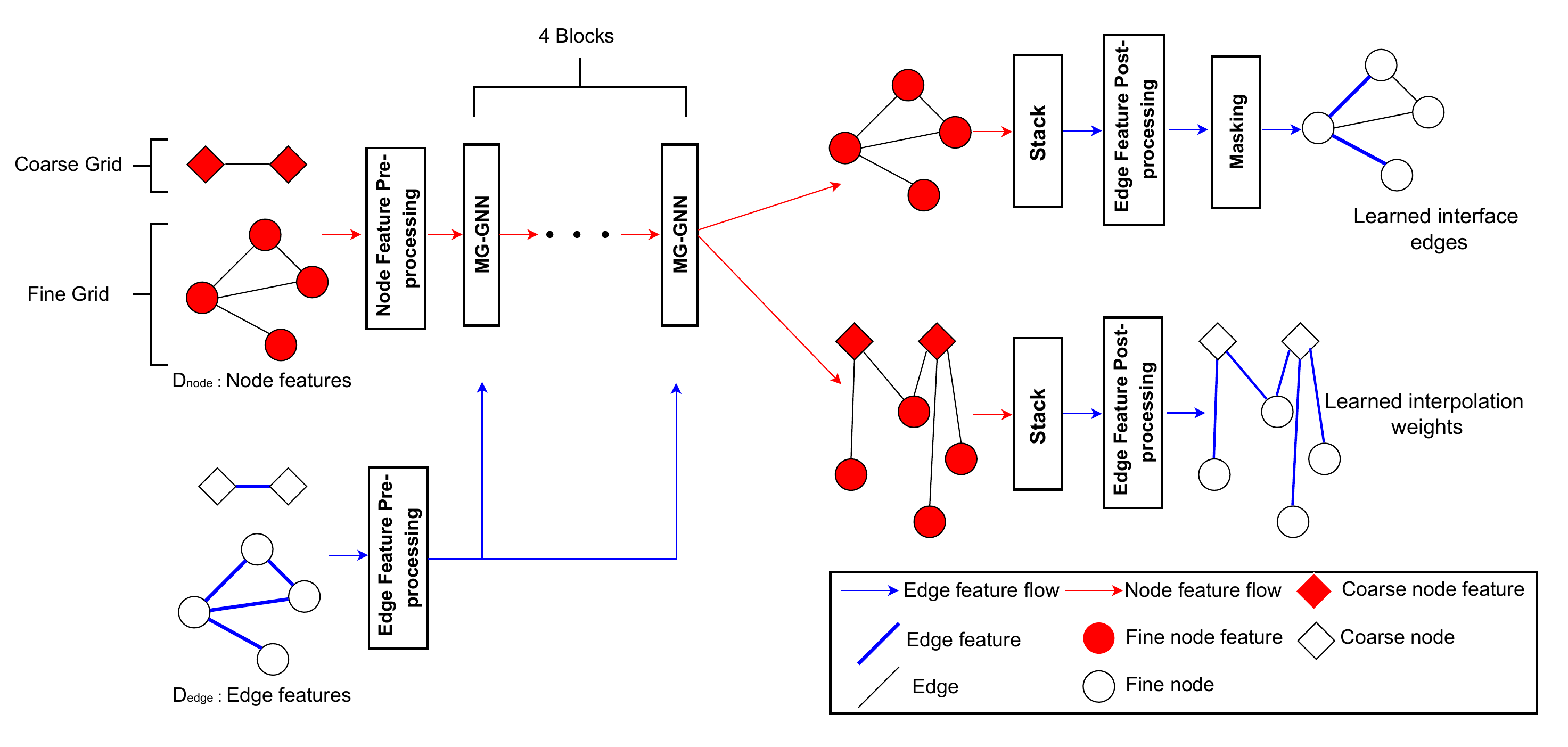}
  \caption{GNN architecture used in this study.}\label{fig:overallgnn}
\end{figure}

\begin{figure}[H]
  \centering
  \includegraphics[width=0.7\textwidth]{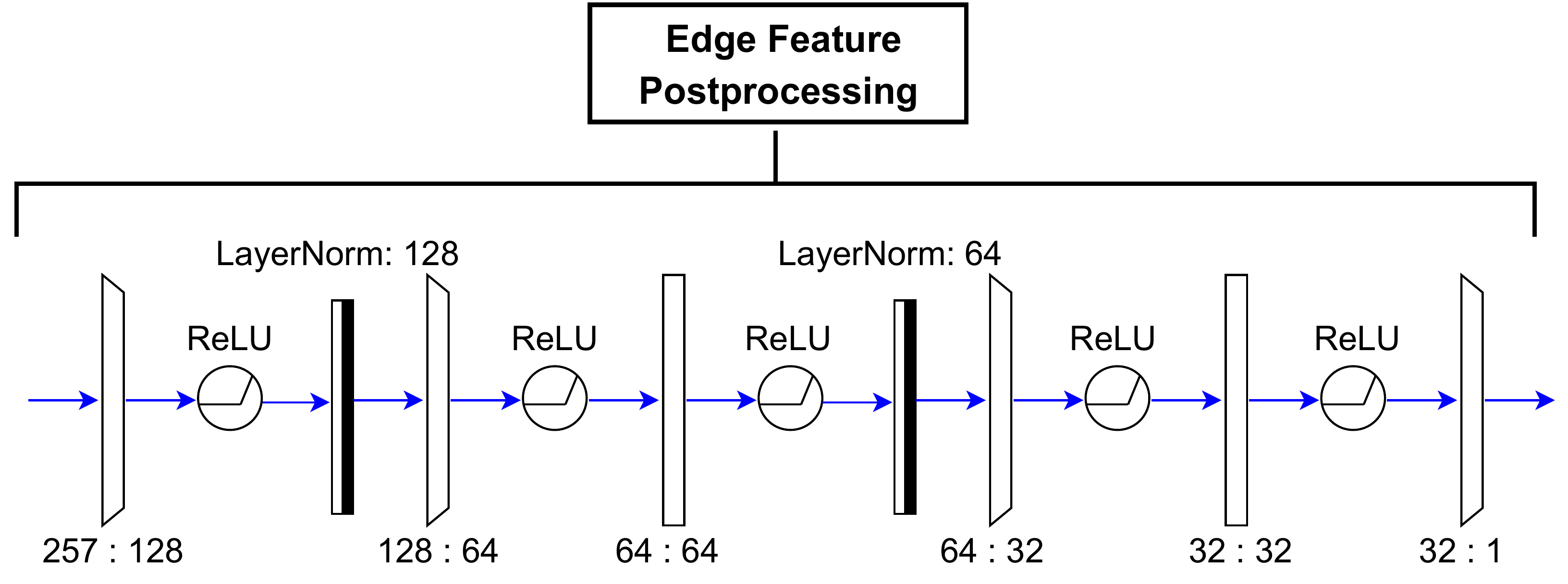}
  \caption{Edge feature post-processing block.}\label{fig:Edgepostprocessing}
\end{figure}
\end{document}